\documentclass[sigconf]{acmart}

\AtBeginDocument{%
  }

\copyrightyear{2026}
\acmYear{2026}
\setcopyright{cc}
\setcctype{by}
\acmConference[WWW '26] {Proceedings of the ACM Web Conference 2026}{April 13--17, 2026}{Dubai, United Arab Emirates.}
\acmBooktitle{Proceedings of the ACM Web Conference 2026 (WWW '26), April 13--17, 2026, Dubai, United Arab Emirates}
\acmISBN{979-8-4007-2307-0/2026/04}
\acmDOI{10.1145/3774904.3792205}
\usepackage{multirow}
\usepackage{bm}
\usepackage{svg}
\usepackage{subcaption}
\usepackage{stfloats}
\usepackage{xspace}
\usepackage{enumitem}
\usepackage{amsthm}
\usepackage{pifont}
\usepackage{makecell}
\usepackage{url}

\newcommand{\stitle}[1]{\vspace*{0.4em}\noindent{\bf #1\/}}
\newcommand{\squishlist}{
	\begin{list}{$\bullet$}
		{ \setlength{\itemsep}{1pt}
			\setlength{\parsep}{1pt}
			\setlength{\topsep}{2.5pt}
			\setlength{\partopsep}{0.5pt}
			\setlength{\leftmargin}{1em}
			\setlength{\labelwidth}{1em}
			\setlength{\labelsep}{0.6em}
		}
	}
	\newcommand{\squishend}{
	\end{list}
}

\newtheorem{theorem}{Theorem}

\begin{document}

\title{Text-attributed Graph Condensation via Text Selection and Attribute Matching}

\author{Haowei Han}
\authornote{Both authors contributed equally to this research.}
\orcid{0009-0002-5201-3893}
\affiliation{%
  \department{School of Computer Science}
  \institution{Wuhan University}
  \city{Wuhan}
  \country{China}
}
\email{haowei.han@whu.edu.cn}

\author{Yuxiang Wang}
\authornotemark[1] 
\orcid{0000-0002-5483-8322}
\affiliation{%
  \department{School of Computer Science}
  \institution{Wuhan University}
  \city{Wuhan}
  \country{China}
}
\email{nai.yxwang@whu.edu.cn}

\author{Guojia Wan}
\orcid{0000-0003-3151-8606}
\affiliation{%
  \department{School of Computer Science}
  \institution{Wuhan University}
  \city{Wuhan}
  \country{China}
}
\email{guojiawan@whu.edu.cn}

\author{Hao Wang}
\orcid{0000-0003-2129-2148}
\affiliation{%
  \department{School of Computer Science}
  \institution{Wuhan University}
  \city{Wuhan}
  \country{China}
}
\email{wanghao.cs@whu.edu.cn}

\author{Shanshan Feng}
\orcid{0000-0002-6161-9232}
\affiliation{%
  \department{School of Computer Science}
  \institution{Wuhan University}
  \city{Wuhan}
  \country{China}
}
\email{victor_fengss@whu.edu.cn}

\author{Hao Huang}
\orcid{0000-0002-3777-1488}
\affiliation{%
  \department{School of Computer Science}
  \institution{Wuhan University}
  \city{Wuhan}
  \country{China}
}
\email{haohuang@whu.edu.cn}

\author{Jiawei Jiang}
\authornote{Corresponding authors.}
\orcid{0000-0003-0051-0046}
\affiliation{%
  \department{School of Computer Science}
  \institution{Wuhan University}
  \city{Wuhan}
  \country{China}
}
\email{jiawei.jiang@whu.edu.cn}

\author{Xiao Yan}
\authornotemark[2] 
\orcid{0000-0002-2122-915X}
\affiliation{%
  \department{Institute for Math \& AI}
  \institution{Wuhan University}
  \city{Wuhan}
  \country{China}
}
\email{yanxiaosunny@whu.edu.cn}

\renewcommand{\shortauthors}{Haowei Han et al.}

\begin{abstract}
    Text-Attributed Graph (TAG) is an important type of graph structured data, where each node has a text description. TAG models usually train a Graph Neural Network (GNN) and language model jointly, which leads to high space and time consumption, especially on large datasets. To mitigate this, we propose TAGSAM, a condensation method that compresses TAGs while preserving training accuracy. TAGSAM comes with two key designs, i.e., \textit{subgraph text Selection} and \textit{Attribute similarity Matching}, which compress the text description and graph topology of TAG, respectively. For the texts, subgraph text selection selects and merges representative text chunks from multiple related text descriptions by maximizing mutual information. For the graph topology, popular condensation methods based on Matching Training Trajectories (MTT) suffer from high variance, which hinders accuracy. Our attribute similarity matching mitigates this issue by aligning stable similarity matrices. We evaluate TAGSAM against six state-of-the-art baselines, where it showcases superior performance. For the same compressed size, TAGSAM improves upon the best-performing baseline by an average of 4.9\% in accuracy. Furthermore, it maintains competitive training accuracy even when the TAG is condensed to just 1\% size. Our code is available at https://github.com/SundayVHan/TAGSAM
\end{abstract}

%
%



\maketitle

\section{Introduction}


Text-Attributed Graphs (TAGs)~\cite{chang2009relational, htag} integrate graph topology with textual node descriptions. Current methods jointly train a Graph Neural Network (GNN) and a language model; for example, G2P2~\cite{wen2023augmenting} aligns graph and text encoder representations via contrastive learning, while GraphCLIP~\cite{zhu2024graphclip} utilizes an LLM to summarize local neighborhoods for richer context. Applied to tasks like article classification and social network analysis~\cite{li2023grenade, brannon2023congrat, exploiting, tesa, retrofitting}, TAG models leverage text as an extra supervisory signal to lower label dependency and boost zero-shot capabilities compared to traditional graph models. However, this dual-modal training remains highly memory- and time-intensive on large datasets.


Dataset condensation~\cite{wang2018dataset, zhao2023dataset} reduces training costs by compressing a large dataset into a smaller, highly representative one. Current approaches fall into two categories.
Selective methods~\cite{har2006maximum,welling2009herding,farahani2009facility, braverman2022power} curate a representative subset from the original data; for instance, Herding~\cite{welling2009herding} minimizes the feature discrepancy between the subset and the full dataset.
Generative methods~\cite{wang2018dataset,zhao2020dataset, zhao2023dataset} synthesize compressed data via bi-level optimization, aligning student models (trained on compressed data) with teacher models (trained on original data). This paradigm has recently been extended to graph by methods like GCond~\cite{jin2021graph}, which matches gradients between student and teacher models to generate synthetic graphs.

\begin{figure}[t]
    \centering
    \includegraphics[width=\linewidth]{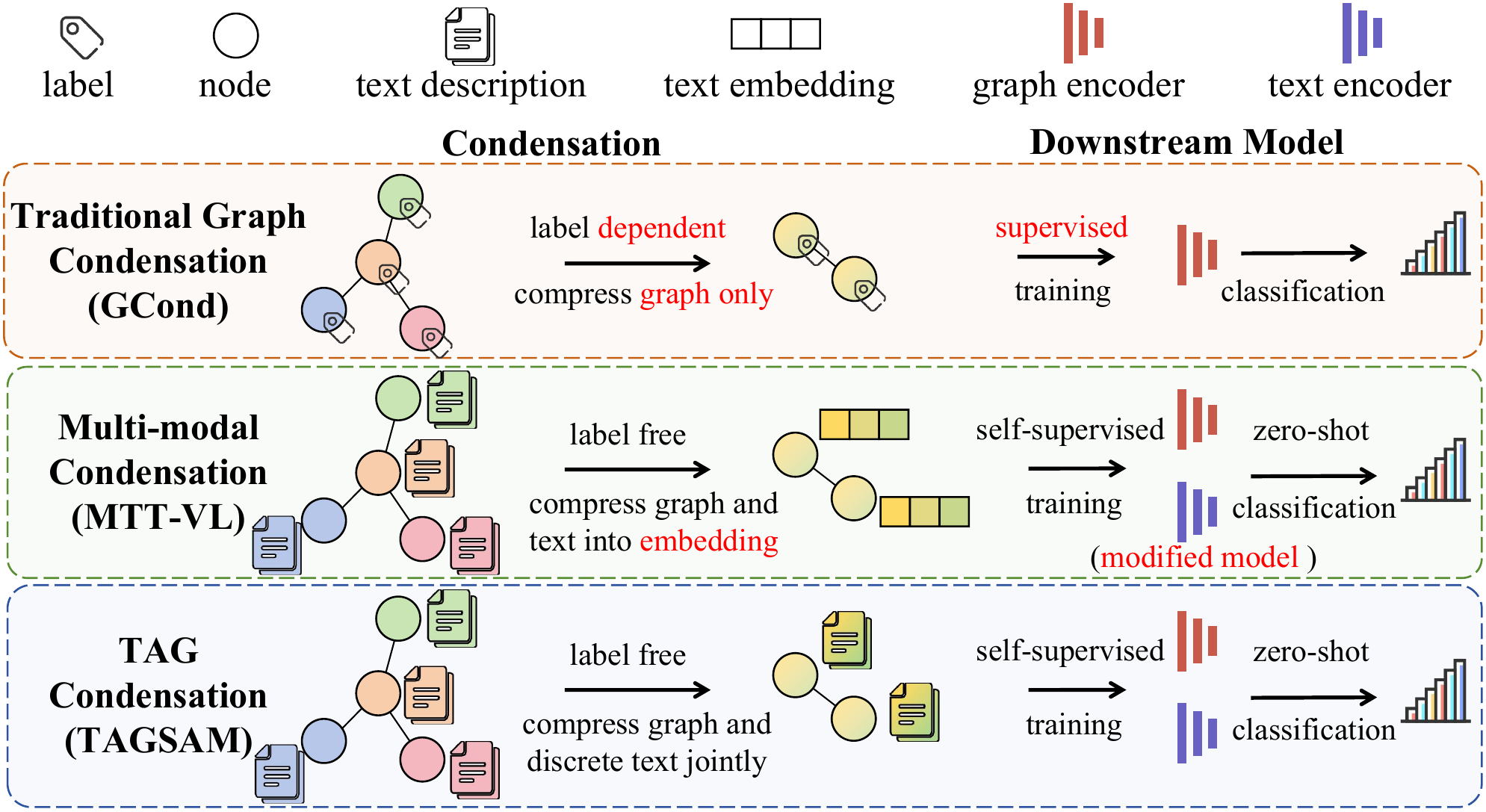}
    \caption{Comparison of different dataset condensation paradigms. Unlike existing methods, TAGSAM ensures label independence and preserves the discrete nature of text data.
    }
    \label{fig:comparison}
    \vspace{-8pt}
\end{figure}

\stitle{Comparison to traditional graph condensation.}
    While effective for standard graphs, traditional graph condensation methods like GCond is ill-suited for TAGs. As Figure~\ref{fig:comparison} demonstrates, such paradigm face two key limitations in this context: First, they condense only the graph modality, ignoring the rich, high-dimensional text. Second, their reliance on label supervision is incompatible with common TAG scenarios, where ground-truth labels are often scarce or entirely unavailable due to high annotation costs~\cite{wen2023augmenting}. Consequently, these traditional methods, not being designed to preserve joint information, fail to meet the needs of TAG models for tasks like contrastive learning and zero-shot classification.

\stitle{Challenges for current condensation methods.} 
    Among recent advances, the work most relevant to TAG condensation is the line of multi-modal dataset condensation methods that employ Matching Training Trajectories (MTT) to compress image-text datasets \cite{wu2023multimodal, xu2024low}, which involves aligning the model parameter update trajectories produced between training on original dataset and compressed one. Yet, despite their success elsewhere, we find they perform poorly on TAGs due to the following two reasons:

    \squishlist
        \item \textit{Unreadable text.} MTT employs gradient descent for condensation, which excels continuous features but struggles with discrete text, resulting in text being compressed into unreadable embedding vectors. Consequently, these embeddings are incompatible with downstream models like tokenizers. While retrieving nearest neighbor original texts in embedding space as substitutes is attempted, it incurs long runtime and lower model performance.
        
        \item \textit{Unstable training.} As shown in Figure~\ref{fig:loss and acc}(a), MTT exhibits significantly volatile training loss and accuracy when condensing a TAG, consequently resulting in sub-optimal convergence. This instability arises because the teacher model's contrastive loss, computed on randomly sampled mini-batches, yields high trajectory variance, and hinders the student model's matching process.
    \squishend

\begin{figure}[t]
    \centering
    \includegraphics[width=\linewidth]{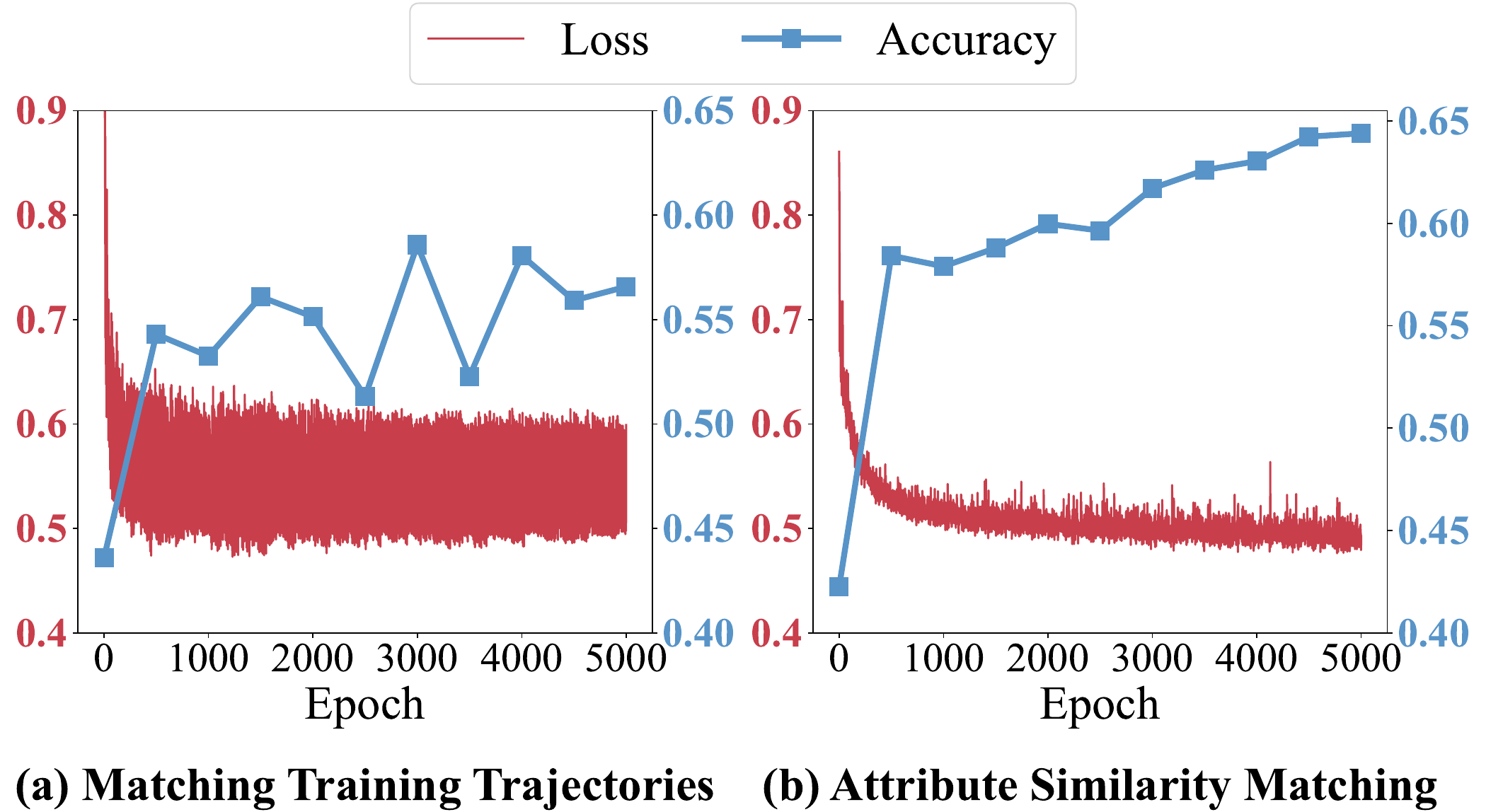}
    \caption{Matching training trajectories vs. Attribute similarity matching. The left and right y-axes denote the model loss and accuracy, respectively.}
    \label{fig:loss and acc}
    \vspace{-8pt}
\end{figure}

These analysis shows that existing condensation methods cannot work well on TAGs. Thus, we ask the following research question:  

\begin{quote}
    \textit{Is it possible to design a TAG condensation method in no-label scenarios, which can generate discrete compressed text and achieves stable convergence?}
\end{quote}

To answer this question, we propose TAGSAM, which tackles the two critical challenges of MTT through two key components: \textit{subgraph text selection} and \textit{attribute similarity matching}.
Instead of updating compressed text embeddings via SGD, subgraph text selection directly selects representative text from the original space.
Besides, compared to condensing TAGs via matching training trajectories, which suffers from high variance due to random batching, we propose to match attribute similarity, which is inherently more stable. This approach yields greater training stability and improves model accuracy, as shown in Figure~\ref{fig:loss and acc}(b).
Furthermore, we theoretically analyze our proposed method. First, for subgraph text selection, we introduce a practical greedy method and theoretically demonstrate that it effectively approximates the optimal solution. Second, we prove that attribute similarity serves as an effective estimate for point-wise mutual information between attributes, and is thereby inherently stable, leading to more robust training.

To evaluate TAGSAM, we conduct extensive experiments and compare it with 6 state-of-the-art dataset condensation methods. The results show that TAGSAM consistently yields higher model training accuracy than all baselines. Compared with the best performing baseline, the average accuracy improvement of TAGSAM is 4.9\%. Moreover, TAGSAM achieves competitive training accuracy with only a 1\% condensation ratio of original TAGs. Besides, ablation studies indicate our model designs are effective. Furthermore efficiency studies suggest that TAGSAM is efficient in running time.

To summarize, we make the following contributions:

\squishlist
    \item We identify that traditional graph condensation methods are unsuitable for TAG condensation scenarios, and existing multi-modal dataset condensation methods perform poorly due to issues such as unreadable text and high variance.
    
    \item We propose subgraph text selection to select representative texts in the discrete text space, and introduce attribute similarity matching to reinforce robustness via matching stable attribute correlations instead of training trajectories.
    
    \item We conduct extensive experiments on 5 datasets and compare TAGSAM with 6 state-of-the-art baselines. On average, TAGSAM raises downstream models' accuracy by 4.9\% over the strongest baseline, and matches or even improves the running time.
\squishend

\section{Preliminaries}
\stitle{Contrastive learning for TAG.} 
    Recent TAG models leverage contrastive learning to jointly learn graph and text embeddings. The goal is to maximize the similarity of positive pairs while minimizing that of negative pairs. For example, G2P2~\cite{wen2023augmenting} employs a GCN~\cite{kipf2016semi} and a transformer~\cite{vaswani2017attention}-based text encoder to align representations of graph topology and node text. A common objective is InfoNCE~\cite{oord2018representation}, where the loss for a batch $B$ of $b$ nodes is defined as:
    \begin{equation}
        \ell_{con}(B) = -\frac{1}{b} \sum_{i=1}^{b} \left[ \frac{\exp(s_{ii} / \tau)}{\sum_{j} \exp(s_{ij} / \tau)} + \frac{\exp(s_{ii} / \tau)}{\sum_{j} \exp(s_{ji} / \tau)} \right],
        \label{equ:contrastive loss}
    \end{equation}
    where $s_{ij}$ is the cosine similarity between node $i$’s graph embedding and node $j$’s text embedding, and $\tau$ is the temperature coefficient. 

\stitle{Zero-shot classification for TAG.}
The contrastive learning leverages text as additional supervisory signal, which enables zero-shot classification. This allows the model to predict a node's class without requiring labeled examples during training~\cite{mandal}. This capability is highly valuable in TAGs because labels are expensive to obtain~\cite{wen2023augmenting}. Detailed implementation can be found in Appendix~\ref{apx:zero-shot}.

\stitle{TAG Condensation Task.}
    We introduce TAG condensation to reduce the high computational training cost. The goal is to compress a TAG $\mathcal{T}=(X,A,T)$—with node features $X\in\mathbb{R}^{N\times d}$, adjacency $A\in\mathbb{R}^{N\times N}$, and node texts $T$—into a synthetic TAG $\mathcal{S}=(\hat X,\hat A,\hat T)$ containing only $M \ll N$ nodes. The condensed representation, defined by $\hat X \in \mathbb{R}^{M\times d}$, $\hat A \in \mathbb{R}^{M\times M}$, and readable texts $\hat T$, is optimized to retain the information necessary for effective model training.

\begin{figure*}[t]
    \centering
    \includegraphics[width=0.95\linewidth]{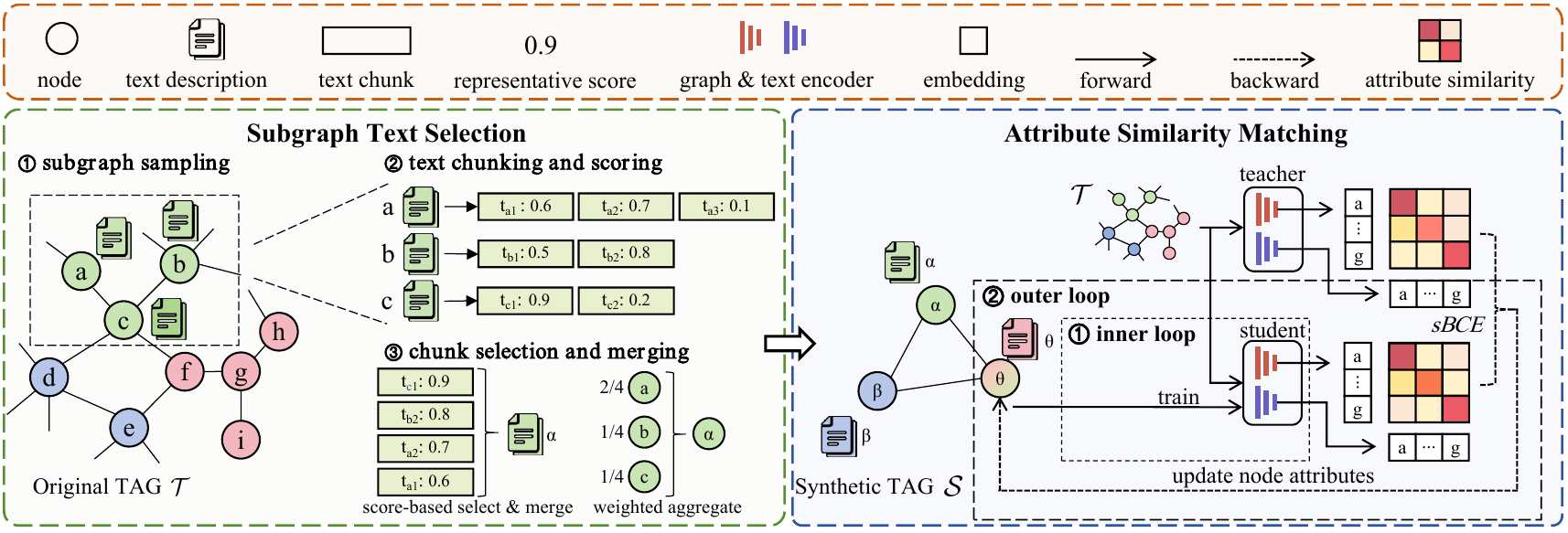}
    \caption{The overview of TAGSAM framework. Subgraph text selection (left) selects the representative chunks and merge them as a compressed text. Attribute similarity matching (right) compresses graph by minimizing the discrepancy between the attribute similarity matrices of the teacher and student models.}
    \label{fig:framework}
\end{figure*}
\section{The TAGSAM Method}

\stitle{Overview.} 
    We propose a dataset condensation method dubbed as TAGSAM to compress the text-attributed graph dataset.
    The framework overview of TAGSAM is illustrated in Figure~\ref{fig:framework}.
    Our method consists of two key condensation techniques: subgraph text selection and attribute similarity matching, which are used to compress text and graph, respectively.
    In particular, we start by introducing the subgraph text selection, which compresses multiple text descriptions into a concise one by selecting representative chunks. Subsequently, we compress graph by stable attribute similarity matching, which avoids the high variance caused by MTT.

\subsection{Subgraph Text Selection}
    \stitle{Problem: unreadable text.}
        Existing multi-modal dataset condensation methods optimize text directly in the embedding space via gradient descent. However, this approach has two critical limitations. First, the optimized embeddings are unreadable. Second, these embeddings are cannot be used as input for the tokenizer, rendering them incompatible with downstream  language models and thus limiting their application.
        Therefore, we ask: \textit{can we compress large-scale text descriptions into concise and readable texts?}
    
    \stitle{Solution: discrete text selection.} 
        To answer the research question, we propose subgraph text selection to compress several text descriptions into a concise one $\hat{T_i}$. The insight behind subgraph text selection is to select and merge representative text chunks from the original  text descriptions, thereby eliminating the need to optimize continuous text embeddings. As depicted in Figure~\ref{fig:framework}~(left), the subgraph text selection process involves three key steps:
        
        \begin{itemize}[leftmargin=*]
             \item{\textbf{Step~\ding{192}: subgraph sampling.}} 
             We sample \(M\) subgraphs by randomly selecting a central node and performing a Breadth-First Search (BFS) until each subgraph reaches a size of \(g\).
            
            \item{\textbf{Step~\ding{193}: text chunking and scoring.} }
            We divide the texts within each subgraph \(G_{\text{sub,i}}\) into a set of chunks, where each chunk consists of a series of consecutive sentences. We denote the collection of all such chunks as \(D_i\). Then, we compute a representative score for each text chunk within each \(D_i\).
            
            \item{\textbf{Step~\ding{194}: chunk selection and merging.}} 
            From the set \(D_i\), we select top \(k\) representative chunks and merge them to form a concise, compressed text $\hat{T_i}$. This process is repeated for each subgraph, resulting in \(M\) compressed text descriptions.
        \end{itemize}
        
        Subgraph text selection helps avoid the risk of blending unrelated texts by merging adjacent texts. Furthermore, compared to using the original texts directly as new text descriptions, our method provides the model with substantially better accuracy through enriched, non-redundant semantic information.
    
    \stitle{Selection criterion.}
        After completing the aforementioned three condensation steps, we can preliminarily compress the text descriptions into a concise one. However, one issue remains unresolved: \textit{what is the criterion for selecting text chunks?} To address this problem, we introduce Chunk-wise Mutual Information (CMI) to evaluate the representativeness of text chunks and select the chunk with the highest representative score for merging at each condensation step. The essence of the CMI is to measure the amount of shared information between different text chunks. Specifically, given two text chunks $c_i$ and $c_j$ within a subgraph $G_{\text{sub,o}}$, their mutual information is calculated as follows:
        \begin{equation}
            {\tt CMI}(c_i; c_j) = \log \frac{p(c_i, c_j)}{p(c_i) \cdot p(c_j)} = \log \frac{p(c_i | c_j)}{p(c_i)}.
            \label{equ:cmi}
        \end{equation}
        After calculating the CMI between each pair of chunks, the next step is to select chunks to form a representative \(\hat{T_o}\). Our design principle is to maximize the \textit{information coverage} of \(D_o\) while minimizing \textit{redundancy}. We define the representativeness score (RS) as follows:
        \begin{equation}
            {\tt RS}(\hat{T_o}) = \underbrace{\sum_{\hat{c} \in \hat{T_o}} \sum_{c \in D_o} {\tt{CMI}}(\hat{c}; c)}_{\text{Information Coverage}} - \underbrace{\sum_{\hat{c_i}, \hat{c_j} \in \hat{T_o}} {\tt{CMI}}(\hat{c_i}; \hat{c_j})}_{\text{Redundancy Penalty}}.
            \label{equ:selection score}
        \end{equation}
        The first term maximizes the information correlation between \(\hat{T_o}\) and \(D_o\), while the second term penalizes redundancy in $\hat{T_o}$.

        However, it is impractical to calculate the real probability distribution of chunks as shown in Equation \ref{equ:cmi}. Therefore, we estimate the probability \( p(c_i) \) of chunk $c_i$ using a pre-trained causal language model \(\phi\). For conditional probabilities \( p(c_i|c_j) \), we estimate it by calculating the probability of \( c_i \) when \( c_j \) is used as a prefix with \(\phi\). Finally, we can estimate the chunk-wise mutual information between \(c_i\) and \(c_j\) as follow: 
        \begin{equation}
            \widehat {\tt CMI}(c_i; c_j) = \frac{1}{2}\max \left( \log \frac{\phi(c_i | c_j)}{\phi(c_i)} + \log \frac{\phi(c_j| c_i)}{\phi(c_j)}, 0 \right).
            \label{merge objective}
        \end{equation}
        We take the average because language model can not guarantee $\log \frac{\phi(c_i| c_j)}{\phi(c_i)} = \log \frac{\phi(c_j| c_i)}{\phi(c_j)}$. Zero truncation is applied to focus on chunk correlation and prevent misleading results from language model instability when taking logarithms.
    
        Since the objective in Equation \ref{equ:selection score} is a combinatorial problem, solving it becomes increasingly costly as \(\left|D_o\right|\) grows larger. To address this problem, we employ an approximate solution by selecting chunks sequentially using a greedy algorithm. In particular, we choose the next chunk from the \(D_o\) that provides the greatest improvement to the Equation \ref{equ:selection score}, until we have selected \( k \) chunks at each condensation step.

    \stitle{Theoretical analysis.}
        We also present a detailed theorem demonstrating that our greedy algorithm, which selects text chunks based on representative score, can approximate the theoretically optimal solution in complex combinatorial optimization.
    
        \begin{theorem}
            Given a cardinality constraint $k$ and the natural logarithm base \( e \), the greedy algorithm yields a solution $\hat{T_o}$ that approximates the optimal text selection $T^*_o$:
            \[{\tt RS}(\hat{T_o}) \geq \left(1 - \frac{1}{e}\right) {\tt RS}(T^*_o).\] 
        \end{theorem}

        The proof is deferred to Appendix~\ref{apx:proof1}.
        It confirms that greedy selection yields at least a constant fraction of the optimum.

\subsection{Attribute Similarity Matching}
    After condensing texts, the next step is to condense the graph, including graph structure and node features.

    \stitle{Graph topology initialization.}
        Through subgraph text selection, we derive \(M\) compressed text descriptions. First, to align each compressed description with a graph structure, we follow GCond~\cite{jin2021graph} and employ an MLP trained on the original graph to predict links based on text embeddings between node pairs. This avoids nested optimization and is highly efficient, as the texts remain fixed.
        Next, to assign a corresponding node feature to each compressed text description, we perform a weighted aggregation of the original node features within each raw subgraph, as shown in step \ding{194} of Figure~\ref{fig:framework}. The weight for each node's features is determined by the proportion of text chunks selected from it.

    \stitle{Bi-level optimization for node feature.}
        While generating node features through aforementioned weighted aggregation is a feasible approach, this simple operation struggles to achieve robust alignment with the compressed texts and the generated graph structure. To overcome this limitation, we adopt the well-established paradigm of bi-level optimization. This allows us to dynamically generate node features.
        We can detail this approach as follows:
    \squishlist
        \item{\bf Inner loop.}
            As shown in Figure~\ref{fig:framework}, we optimize a model \(f\) called student model with randomly initialized parameter \(\theta^0_s\) by continuously reducing \(\ell_{con}\) on the synthetic graph \(\mathcal{S}\) with learning rate \(\eta\) as shown in Equation \ref{inner}.
            \begin{equation}
            \theta_s^{t+1}  \leftarrow \theta_s^{t}- \eta \nabla_{\theta} \ell_{con} \left( f_{\theta_s^t} \left(\mathcal{S}\right)\right)
            \label{inner}
            \end{equation}
        \item{\bf Outer loop.}
            Let \(\Theta = \{\theta^0,\theta^1,...,\theta^*\}\) denote the set of parameters updated at each step. We aim to find a suitable \(\hat{X}\), such that the set of student parameters \(\Theta_s\) obtained from inner loop, aligns with the set of teacher parameters \(\Theta_t\) obtained from training on the original graph \(\mathcal{T}\). The alignment is measured by a matching loss function \(\ell_{match}\) in a certain aspect. We denote it as follow:
            \begin{equation}
                \min_{\hat{X}}\mathbb{E}_{\theta^0_s \sim P_{\theta^0_s}} \ell_{match} \left( \Theta_s, \Theta_t \right)
            \end{equation}
    \squishend

    \stitle{Problem: high variance in training trajectories.}
        Current multi-modal dataset condensation methods often utilize the above bi-level optimization framework and employ Matching Training Trajectories (MTT) as their matching loss function $\ell_{match}$. This approach operates by aligning the parameter update trajectories of teacher and student models across their respective training stages. However, when applying this strategy to TAG condensation, we observe that the compressed graph often fails to accurately represent the original. The core issue lies in the nature of the teacher models in the TAG scenario, which are pre-trained via contrastive learning. As demonstrated by~\cite{joshi2024dataset}, contrastive learning loss inherently produces gradients with high variance. This is because the loss for any given sample is highly dependent on the other samples within the same batch, unlike in supervised learning where each example's contribution is independent. 
        The high variance in gradients causes unpredictable and unstable teacher trajectories. As a result, the trajectory matching task becomes a very difficult optimization problem. This instability, as shown by our observations in Figure~\ref{fig:loss and acc}, ultimately reduces the effectiveness of MTT for TAG condensation.

    \stitle{Solution: attribute similarity matching.}
        To avoid the unstable optimization process associated with MTT, we design a condensation technique dubbed as \textit{attribute similarity matching}. In particular, the product of graph embeddings $f^G(X)$ and text embeddings $f^T(T)$ between each pair of samples forms the attribute similarity matrices $SM$. We aim to minimize the discrepancy between the similarity matrices computed by the student model (i.e., $SM_s$) and the teacher model (i.e., $SM_t$) on the original samples, ensuring student model can capture the inter-modal relationships. Formally, we denote our matching loss function $\ell_{match}$ as follows:
        \begin{equation}
            SM_s = f^G_{\theta^*_s}(X) f^T_{\theta^*_s}(T)^\top, SM_t = f^G_{\theta^*_t}(X) f^T_{\theta^*_t}(T)^\top .
        \end{equation}
        \begin{equation}
            \ell_{match} = \frac{1}{2}(\ell_{sBCE}(SM_s,SM_t) + \ell_{sBCE}(SM_s^\top,SM_t^\top)),
        \end{equation}
        where \(sBCE\) means that applying a row-wise softmax to each similarity matrix before computing Binary Cross Entropy, giving smoother gradients and steadier alignment.

    \begin{table*}[t]
    \centering
    \caption{Complexity analysis of matching trajectory training (MTT) and our attribute similarity matching (ASM).}
    \begin{tabular}{c|c|cccc}
    \toprule
    \multirow{2}{*}{\textbf{Method}} & \multirow{2}{*}{\textbf{Teacher Model Training}} & \multicolumn{4}{c}{\textbf{Condensation}} \\ \cline{3-6}
     & & Forward $\ell_{con}$& Update $f_s$ & Loss $\ell_{match}$ & Update $\mathcal{S}$ \\ \midrule
    \textbf{MTT} & $\mathcal{O}(Z(L_gEh_g+L_gBh_g^2+L_tBh_t^2+B^2))$ & $\mathcal{O}(q(L_gMh_g^2+L_tMh_t^2))$ & $\mathcal{O}(q(L_gh_g^2+L_th_t^2))$ & $\mathcal{O}(L_gh_g^2+L_th_t^2)$ & $\mathcal{O}(M)$ \\ 
    \textbf{ASM} & $\mathcal{O}(L_gEh_g+L_gBh_g^2+L_tBh_t^2+B^2)$ & $\mathcal{O}(q(L_gMh_g^2+L_tMh_t^2))$ & $\mathcal{O}(q(L_gh_g^2+L_th_t^2))$ & $\mathcal{O}(L_gmh_g^2+L_tmh_t^2+m^2)$ & $\mathcal{O}(M)$ \\ \bottomrule
    \end{tabular}
    \label{tab:complexity}
\end{table*}
    
    \stitle{Theoretical analysis.}
        Following, we provide a detailed explanation of why attribute similarity matching is more robust than matching training trajectories.

        \begin{theorem}
            In TAG contrastive learning, the optimal embeddings $f_G^{*}(X_i)$ and $f_T^{*}(T_i)$ are unique only up to a rotation matrix $R$ followed by a scaling matrix $S$:
            \begin{equation}
                f^{*}_G(X_i) \;=\; \frac{1}{\sqrt{p_X(i)}}\bigl(U_i^{k} S R\bigr)^{\top},
                \label{equ:ideal graph representation}
            \end{equation}
            \begin{equation}
                f^{*}_T(T_i) \;=\; \frac{1}{\sqrt{p_T(i)}}\bigl(V_i^{k}\Sigma^{k} S^{-1} R\bigr)^{\top},
                \label{equ:ideal text representation}
            \end{equation}
            where \(U^{k}, \Sigma^{k}, V^{k}\) are the rank-\(k\) factors of the normalized attribute joint distribution \(\tilde{P}\).
        
            Consequently, the logarithm of the optimal similarity \(s_{ij}^{*}\) between any node attribute \(X_i\) and text attribute \(T_j\) equals their pointwise mutual information:
            \begin{equation}
            \log s_{ij}^{*} = \log\bigl(f^{*}_G(X_i) f^{*}_T(T_j)^{\top}\bigr) = \operatorname{PMI}(X_i, T_j).
            \label{equ:similarity}
            \end{equation}
    
        \label{th:similarity}
        \end{theorem}
        
        \begin{figure}[t]
            \centering
            \includegraphics[width=\linewidth]{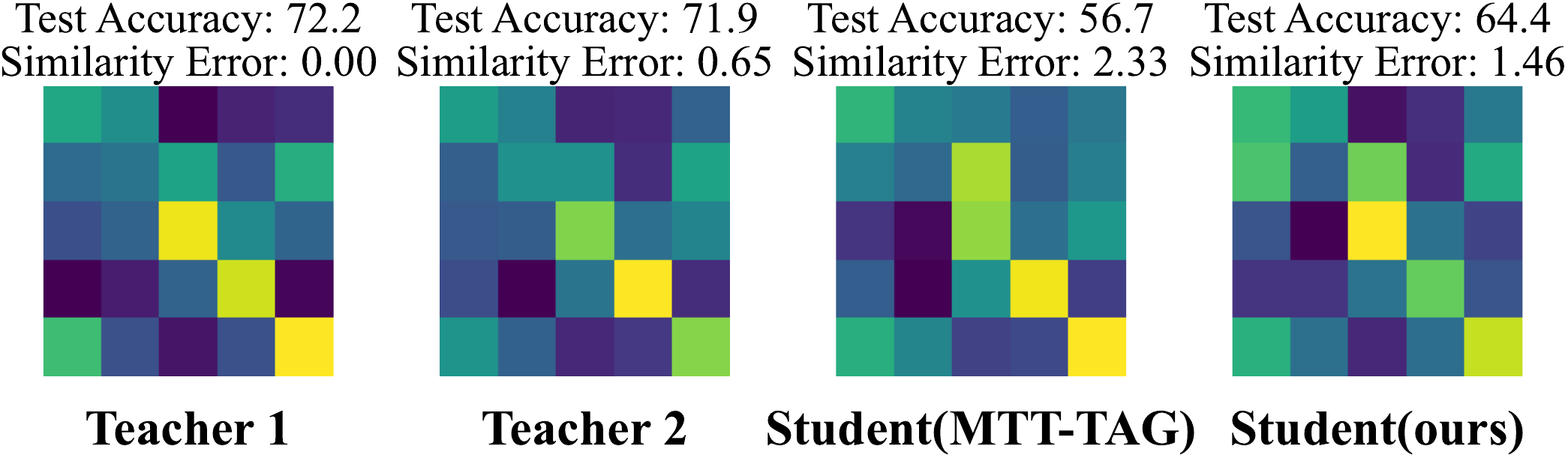}
            \caption{Visualization of attribute similarity matrices for teacher and student models. \textit{Similarity Error} is the element-wise relative error with respect to Teacher 1.}
            \label{fig:similarity}
        \end{figure}

        The proof can be found in Appendix~\ref{apx:proof2}. We observe two insights from  Theorem~\ref{th:similarity} as follows:
        \squishlist
            \item \textit{Unstable training trajectory of MTT.}
            Equations~\ref{equ:ideal graph representation} and \ref{equ:ideal text representation} show that teacher embeddings are defined up to arbitrary rotation \(R\) and scaling \(S\).  
            Because MTT forces the student to follow the teacher’s full training trajectories, the search space becomes very large and optimization turns unstable.  
            Meanwhile, random sampling in contrastive learning keeps these trajectories more fluctuated.
            As \textit{Student (MTT-TAG)} in Figure~\ref{fig:similarity} illustrates, the unstable training leads to lower accuracy.
        
            \item \textit{Robust similarity matrix.} Equation~\ref{equ:similarity} shows that the optimal element \( s_{ij}^* \) of this matrix is equal to the PMI between \( X_i \) and \( T_j \), making them independent of the arbitrary matrices \( R \) and \( S \) in the graph and text encoders. Attribute similarity matching aligns the similarity matrices of the teacher and student models, thereby avoiding directly matching the graph and text encoders. Figure~\ref{fig:similarity} shows that our student model achieves lower error and higher accuracy than the student model in MTT-TAG.
        
        \squishend

    \stitle{Complexity analysis.}
        For an intuitive comparison, we present the detailed time complexity of matching training trajectory (MTT) and attribute similarity matching (ASM) in Table~\ref{tab:complexity}. Let \( E \) be the nodes and edges in the \(\mathcal{T}\), and \( M \) the nodes in the \(\mathcal{S}\). \( Z \) appears only in the teacher training process of MTT, representing the number of teachers. In the inner loop, the student model undergoes \( q \) epoch updates. Additionally, The model's graph encoder has \( L_g \) layers with \( h_g \) hidden dimension, and the text encoder comprises \( L_t \) layers with \( h_t \) hidden dimension. $B$ and $m$ denote the batch sizes for teacher training and similarity matching, respectively. 

        According to Table~\ref{tab:complexity}, the overall time complexity of both methods is largely similar, except for the ``Teacher Model Training'' phase. MTT's approach requires training Z (typically ~15) separate teacher models for collecting varied training trajectories, making its cost prohibitive for large-scale graphs. In contrast, by matching stable similarity matrices, our ASM method needs only one teacher. This fundamental difference makes ASM a significantly more efficient and scalable solution for TAG condensation.

\section{Experimental Evaluations}




\begin{table*}[t]
\caption{Classification accuracy in percentage (\%) for models trained on the condensed datasets (mean±std).
Best results are in bold. \textit{Whole} refers to training on the original graph. Numbers below datasets indicate the original node count, while \textit{Size} denotes the compressed node count. We calculate the accuracy increment by comparing against the best baseline in each setting.
}
\centering

\begin{tabular}{c|c|ccc|c|ccl|c}
\toprule
\textbf{Dataset} & \textbf{Size} & \textbf{Random} & \textbf{Herding} & \textbf{K-Center} & \textbf{GCond*} & \textbf{MTT-TAG} & \textbf{MTT-KNN} & \textbf{TAGSAM (Ours)} & \textit{Whole} \\ \toprule
& 50 & $50.0 \pm 2.5$ & $33.4 \pm 2.7$ & $47.3 \pm 2.6$ & $40.1 \pm 3.9$ & $64.4 \pm 3.4$ & $62.8 \pm 1.6$ & \bm{$73.8 \pm 1.8 (\uparrow9.4)$} & \\
& 100 & $68.1 \pm 1.4$ & $34.4 \pm 2.1$ & $61.9 \pm 1.1$ & $49.7 \pm 3.6$ & $73.8 \pm 0.9$ & $71.1 \pm 1.7$ & \bm{$75.2 \pm 1.1 (\uparrow1.4)$} & \\
\multirow{-3}{*}{\parbox{1.8cm}{\centering Cora \\ (2,708)}} & 200 & $71.8 \pm 1.2$ & $43.8 \pm 0.9$ & $68.3 \pm 1.1$ & $55.2 \pm 2.9$ & $76.1 \pm 0.5$ & $72.4 \pm 1.2$ & \bm{$76.6 \pm 0.8 (\uparrow0.5)$} & \multirow{-3}{*}{$78.7 \pm 0.5$} \\ \midrule
& 100 & $41.4 \pm 2.1$ & $20.8 \pm 1.1$ & $40.8 \pm 3.1$ & $38.5 \pm 4.3$ & $51.7 \pm 3.4$ & $49.8 \pm 2.7$ & \bm{$56.4 \pm 2.3 (\uparrow4.7)$} & \\
& 200 & $50.4 \pm 2.3$ & $29.2 \pm 2.1$ & $45.3 \pm 3.4$ & $43.3 \pm 3.9$ & $56.2 \pm 0.9$ & $52.8 \pm 0.8$ & \bm{$60.8 \pm 1.9 (\uparrow4.6)$} & \\
\multirow{-3}{*}{\parbox{1.8cm}{\centering Photo \\ (48,362)}} & 500 & $59.1 \pm 1.9$ & $33.5 \pm 0.6$ & $53.6 \pm 2.1$ & $45.4 \pm 4.2$ & $59.1 \pm 0.6$ & $56.9 \pm 0.9$ & \bm{$63.3 \pm 0.9 (\uparrow4.2)$} & \multirow{-3}{*}{$63.3 \pm 0.2$} \\ \midrule
& 100 & $32.4 \pm 3.3$ & $21.5 \pm 3.8$ & $42.4 \pm 3.8$ & $43.8 \pm 4.9$ & $48.0 \pm 1.7$ & $37.3 \pm 3.1$ & \bm{$62.1 \pm 0.6 (\uparrow14.1)$} & \\
& 200 & $46.4 \pm 3.0$ & $30.4 \pm 1.7$ & $45.4 \pm 2.1$ & $48.4 \pm 2.2$ & $56.7 \pm 1.6$ & $56.7 \pm 2.9$ & \bm{$64.4 \pm 0.3 (\uparrow7.7)$} & \\
\multirow{-3}{*}{\parbox{1.8cm}{\centering Computer \\ (87,229)}} & 500 & $49.8 \pm 2.5$ & $35.3 \pm 2.5$ & $49.0 \pm 3.1$ & $50.1 \pm 2.0$ & $61.6 \pm 0.9$ & $59.8 \pm 1.2$ & \bm{$70.6 \pm 0.4 (\uparrow9.0)$} & \multirow{-3}{*}{$72.2 \pm 0.3$} \\ \midrule
& 100 & $46.3 \pm 1.1$ & $28.3 \pm 1.4$ & $45.2 \pm 1.1$ & $27.7 \pm 3.1$ & $22.7 \pm 0.9$ & $20.3 \pm 3.3$ & \bm{$49.6 \pm 1.1 (\uparrow3.3)$} & \\
& 200 & $47.5 \pm 1.2$ & $29.7 \pm 0.6$ & $47.0 \pm 0.7$ & $35.8 \pm 3.4$ & $30.5 \pm 1.6$ & $28.8 \pm 3.4$ & \bm{$51.9 \pm 1.5 (\uparrow4.4)$} & \\
\multirow{-3}{*}{\parbox{1.8cm}{\centering Arxiv \\ (169,343)}} & 500 & $52.7 \pm 0.5$ & $32.7 \pm 0.4$ & $52.8 \pm 0.5$ & $42.3 \pm 1.9$ & $49.5 \pm 0.5$ & $48.1 \pm 2.3$ & \bm{$55.4 \pm 0.6 (\uparrow2.6)$} & \multirow{-3}{*}{$55.3 \pm 0.3$} \\ \midrule
& 500 & $52.7 \pm 2.3$ & $17.8 \pm 1.4$ & $26.3 \pm 2.4$ & $39.0 \pm 1.8$ & $48.4 \pm 1.6$ & $46.2 \pm 3.2$ & \bm{$55.4 \pm 1.0 (\uparrow2.7)$} & \\
& 1000 & $53.8 \pm 1.8$ & $25.8 \pm 0.7$ & $29.9 \pm 2.3$ & $45.8 \pm 2.0$ & $50.0 \pm 2.4$ & $47.6 \pm 2.1$ & \bm{$55.8 \pm 0.6 (\uparrow2.0)$} & \\
\multirow{-3}{*}{\parbox{1.8cm}{\centering Products \\ (2,449,029)}} & 2000 & $54.8 \pm 2.0$ & $33.7 \pm 0.7$ & $38.5 \pm 2.4$ & $47.5 \pm 1.5$ & $53.3 \pm 1.7$ & $49.5 \pm 2.3$ & \bm{$57.3 \pm 0.2 (\uparrow2.5)$} & \multirow{-3}{*}{$57.6 \pm 0.1$} \\ \bottomrule
\end{tabular}
\label{tab:main result}
\end{table*}

\subsection{Experiment Settings}
    \stitle{Datasets.}
        We benchmark TAGSAM on 5 text-attributed graphs.
        \textit{Cora}~\cite{sen2008collective} and \textit{ArXiv}~\cite{wang2020microsoft} are citation networks whose nodes are papers described by titles and abstracts, with edges representing citations.
        \textit{Photo}, \textit{Computers}~\cite{shchur2018pitfalls}, and \textit{Products}~\cite{hu2020open} are Amazon co-purchase graphs; each node is a product whose text attribute is its most‐voted user review, and edges indicate co-purchase interactions.
        These datasets vary widely in scale, from 2.7 K to 2.4 M nodes.
        More details about datasets are provided in Appendix~\ref{apx:dataset detail}.
    
    \stitle{Baselines.}
        We compare TAGSAM with 6 representative baselines.
        
        \squishlist
        \item \textbf{Selective methods.} (a) \textit{Random} selects samples randomly. (b) \textit{Herding}~\cite{welling2009herding} greedily selects samples in order to minimize the distance between the center of the coreset and the center of the full dataset in feature space. (c) \textit{K-Center}~\cite{farahani2009facility} aims to select coreset samples that maximize their dispersion in the feature space.
        
        \item \textbf{Generative methods.} (a) \textit{GCond*} adapts GCond~\cite{jin2021graph} for text-attributed graphs by extending its original graph-only gradient matching to a joint graph-text objective.
        (b) \textit{MTT-TAG} adapts MTT-VL~\cite{wu2023multimodal}, an image-text multi-modal condensation method, by replacing its image encoder with a graph encoder.
        (c) \textit{MTT-KNN} further compresses text into optimized embeddings, then employs a K-Nearest Neighbors (KNN) mapping, a variant approach proposed by MTT-VL, to retrieve the closest original text.
        \squishend
    
    \stitle{Evaluation protocol.} 
        Models are then trained on condensed graph and evaluated on the original test split. Following the five-way zero-shot classification setting of G2P2 \cite{wen2023augmenting}, we report the classification accuracy over 5 independent runs. By default, GCN~\cite{kipf2016semi} and BERT~\cite{devlin2019bert} with a projector serve as the graph and text encoders. More hyperparameter settings can be found in Appendix~\ref{apx:parameter}. 

    \begin{table}[t]
    \centering
    \setlength{\tabcolsep}{2.0mm} 
    \caption{Cross-architecture classification accuracy for models trained on the condensed Computer which is compressed using a combination of GCN and Bert.}
    \begin{tabular}{cc|cccc}
    \toprule
    \multirow{2}{*}{\textbf{Method}} & \multirow{2}{*}{\makecell{\textbf{Text} \\ \textbf{Encoder}}} & \multicolumn{4}{c}{\textbf{Graph Encoder}} \\ \cline{3-6}
                                     &  & \textbf{GCN}  & \textbf{MLP}  & \textbf{SAGE} & \textbf{Cheby}\\ \hline
    Random  &  \multirow{3}{*}{BERT}    & 43.6          & 39.0          & 35.2          & 33.4          \\
    MTT-TAG &                           & 62.1          & 38.0          & 34.8          & 35.0          \\
    TAGSAM  &                           & \textbf{64.4} & \textbf{46.9} & \textbf{45.5} & \textbf{46.2} \\ \hline
    Random  &  \multirow{3}{*}{RoBERTa} & 28.5          & 25.2          & 25.0          & 23.6          \\
    MTT-TAG &                           & 39.3          & 27.5          & 29.7          & 28.3          \\
    TAGSAM  &                           & \textbf{47.7} & \textbf{35.4} & \textbf{35.7} & \textbf{36.3} \\
    \bottomrule
    \end{tabular}
    \label{tab:cross architectures}
\end{table}

\subsection{Main Results}
    \stitle{Classification performance.}
        To assess the quality of the compressed TAGs, Table~\ref{tab:main result} reports the classification accuracy of all baselines and our TAGSAM. We make the following observations:
        
        Selective methods like Herding, and K-Center perform relatively poorly, confirming that simple heuristic sampling without learning cannot effectively preserve information from original TAG. Random sampling sometimes outperforms most other methods, especially on datasets with many classes (e.g. ArXiv and Products). This phenomenon has also been observed in image-text dataset condensation~\cite{wu2023multimodal}. We hypothesize this is because without class guidance, their selection heuristics can introduce significant sampling bias by high-density features that may not correlate with class diversity. Random sampling, in contrast, is less affected as it is agnostic to these potentially misleading data structures.

        GCond* performs poorly because its gradient-based matching strategy is designed to match the gradients of samples with the same class label. However, in TAG condensation, where no class label can be used, this approach is ineffective because matching gradients between different class samples will lead to instability. Furthermore, such conventional graph condensation methods are designed to compress graph modality only, which means they are limited by their inability to compress text modality.
        
        Existing multi-modal condensation methods like MTT-KNN and MTT-TAG generally perform best on average. Notably, the accuracy of MTT-KNN is lower than MTT-TAG. This is because MTT-KNN compresses texts into embedding vectors, for which it has to find the nearest neighbor original text descriptions. Such an operation disrupts the alignment between the compressed node features and texts produced by the learning process.
        Despite their higher performance, they are also limited by their inability to compress discrete text descriptions and suffer from unstable matching objectives. Our proposed TAGSAM overcomes these limitations, achieving superior performance across all datasets and condensation sizes. It boosts accuracy by up to 14.1\% over the strongest baseline MTT-TAG. For larger condensation sizes, TAGSAM's accuracy, although demonstrating more variability compared to training on the Whole dataset, reaches an average value highly comparable to it.

    \stitle{Cross-architecture generalization.}  
        Table~\ref{tab:cross architectures} evaluates cross architecture generalization by condensing the Computer dataset to 200 nodes. We compare our method, TAGSAM, with top baselines: Random for selective condensation and MTT-TAG for generative condensation. Specifically, we use a GCN+BERT architecture to condense the dataset for TAGSAM and MTT-TAG. Subsequently, these small, condensed graphs are used to train various other model combinations, including graph encoders (GCN, MLP, SAGE~\cite{hamilton2017inductive}, Cheby~\cite{defferrard2016convolutional}) and text encoders (BERT~\cite{devlin2019bert}, Roberta~\cite{liu2019roberta}).
        
        The results in Table \ref{tab:cross architectures} suggest that all methods experience accuracy degradation. The degradation also occurs in Random, demonstrating that there is significant accuracy loss caused by the encoder's lower compatibility. For the MTT-TAG and TAGSAM, the degradation is additionally caused by the use of different architectures during the condensation and testing phases.
        However, TAGSAM consistently achieves the highest accuracy due to our effective and stable condensation. Compared to random sampling, TAGSAM shows a substantial average improvement of 13.1\%. 

    \begin{table}[t]
    \centering
    \setlength{\tabcolsep}{3.0mm}
    \caption{Link prediction AUC for models trained on the condensed datasets (mean+std). Best results are in bold.}
    \begin{tabular}{l|ccc}
    \toprule
    Dataset     & Random           & MTT-TAG        & TAGSAM              \\ \toprule
    Cora        & $75.85 \pm 1.32$ & $78.26 \pm 1.34$ & \bm{$82.15 \pm 0.97$} \\
    Photo       & $84.34 \pm 1.43$ & $86.44 \pm 1.41$ & \bm{$87.69 \pm 1.20$} \\
    Computer    & $87.62 \pm 2.15$ & $88.33 \pm 1.74$ & \bm{$91.38 \pm 1.83$} \\
    Arxiv       & $81.93 \pm 2.04$ & $86.35 \pm 1.22$ & \bm{$87.37 \pm 1.57$} \\ 
    Products    & $84.21 \pm 2.20$ & $84.62 \pm 2.39$ & \bm{$86.06 \pm 2.05$} \\ 
    \bottomrule
    \end{tabular}
    \label{tab:link}
\end{table}

    \stitle{Link prediction performance.}        
        We evaluate link prediction by training an MLP on features extracted from the graph encoder. As shown in Table~\ref{tab:link}, our TAGSAM outperforms the top selective (Random) and generative (MTT-TAG) baselines, achieving significant AUC improvements of 4.1\% and 2.1\% respectively. This performance suggests our method more effectively retains structural information from the original graph, thereby enhancing the models trained on the resulting condensed graph.

\begin{figure}[t]
    \centering
    \includegraphics[width=\linewidth]{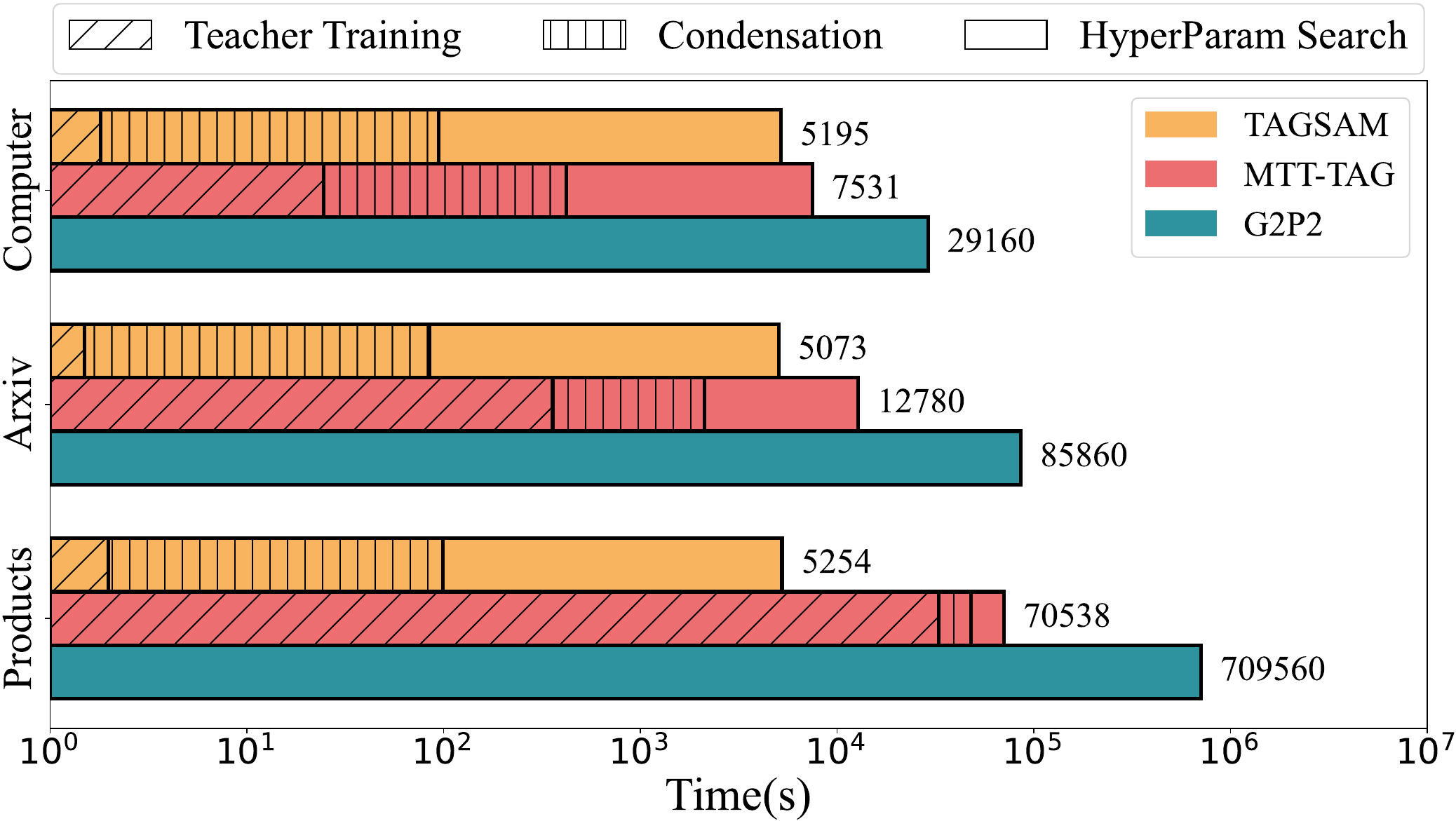}
    \caption{Running time (in log scale) for hyperparameter search of downstream model. G2P2 trains directly on the original graph, while TAGSAM and MTT-TAG first compress the graph and then train models on the condensed graph.}
    \label{fig:efficiency}
\end{figure}
    
    \stitle{Efficiency study.} 
        We study the impact of TAG condensation on accelerating hyperparameter searches for downstream TAG models, and compare running time of TAGSAM, MTT-TAG and G2P2 (direct training). Specifically, we use 108 hyperparameter settings: optimizers $\{\text{SGD},\ \text{Adam},\ \text{AdamW}\}$; graph encoder depths $\{2,3,4\}$; projector depths $\{1,2,3\}$; and embedding sizes $\{64,128,192,256\}$. We also decompose the running time of MTT-TAG and TAGSAM into \textit{teacher training}, \textit{condensation}, and \textit{hyperparameter search}. 

        The results in Figure~\ref{fig:efficiency} show that the larger the dataset, the greater the time savings from condensation.  
        Although condensation incurs a one-time cost, it significantly reduces the time for each subsequent hyperparameter trial, meaning its advantage grows with the number of searches.
        For larger datasets, teacher model training dominates total runtime. TAGSAM requires training only a single teacher due to its attribute similarity consistency (Figure \ref{fig:similarity}), eliminating the need for additional teachers, unlike MTT-TAG which must train multiple teachers to reflect diverse training trajectories.

    \begin{figure}[t]  
        \centering
        \includegraphics[width=0.95\linewidth]{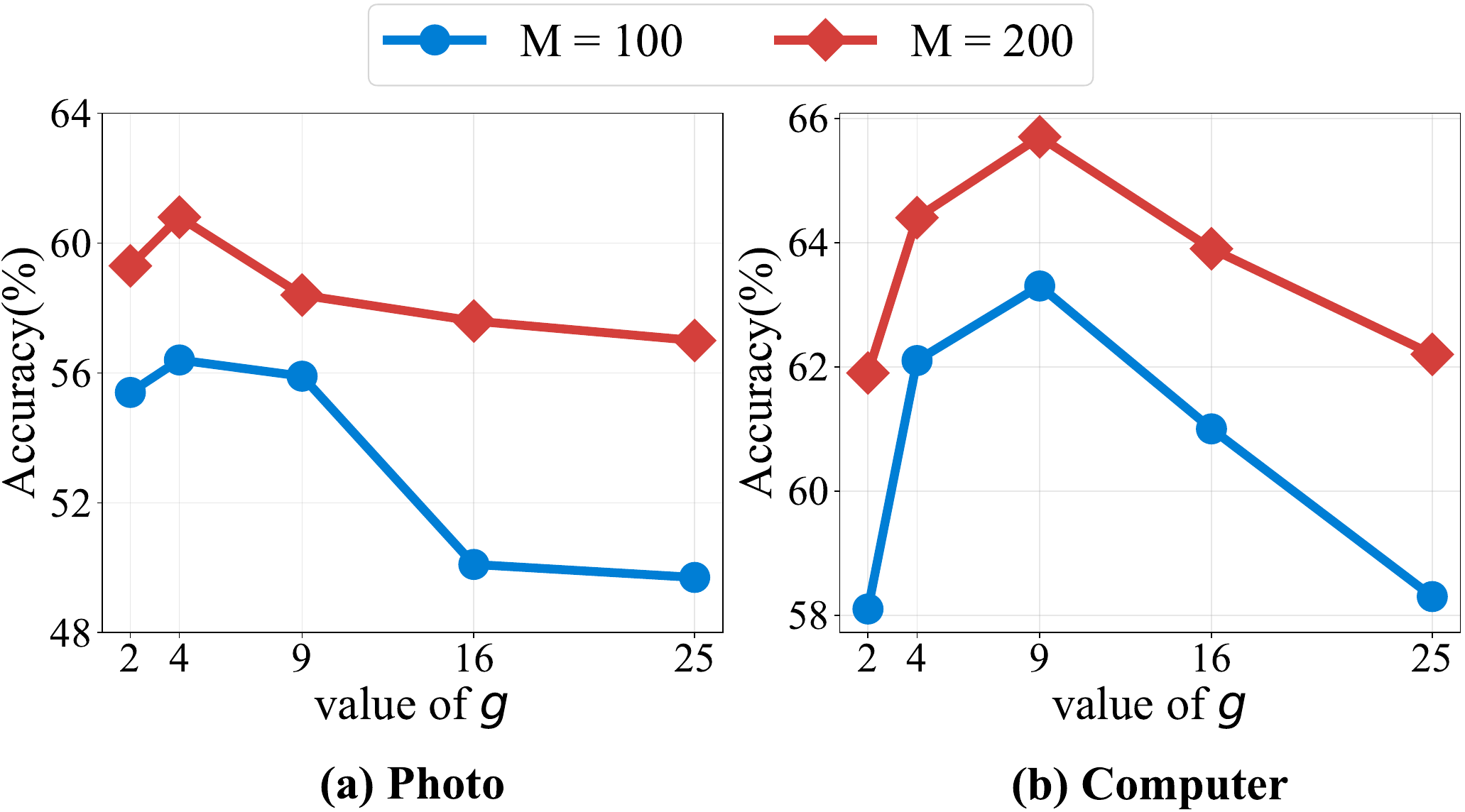}
        \caption{The impact of subgraph size for subgraph text selection. \textit{M} denotes the number of compressed texts/nodes.}
        \label{fig:selection number}
    \end{figure}

    \stitle{Hyperparameter sensitivity.}
         We study the hyperparameter sensitivity experiments with respect to the subgraph size \(g\) in subgraph text selection on the Photo and Computer datasets. The results are shown in Figure \ref{fig:selection number}. A larger \(g\) means more information is integrated, but it can also lead to more information confusion. 
         
         We validate the impact of \( g \) on compression performance through experiments on the Photo and Computer datasets. We find that setting \(g = 4\) and \(g = 9\) yields the best performance for the downstream models. Additionally, we observe a sharp performance drop when \(g \ge 16\). We attribute this decline to the fact that such a large \(g\) involves selecting many 2-hop neighbors, which introduces semantically disparate texts, making it difficult to optimize the alignment of the corresponding node features for the newly compressed text.

    \begin{figure}[t]
        \centering
        \includegraphics[width=\linewidth]{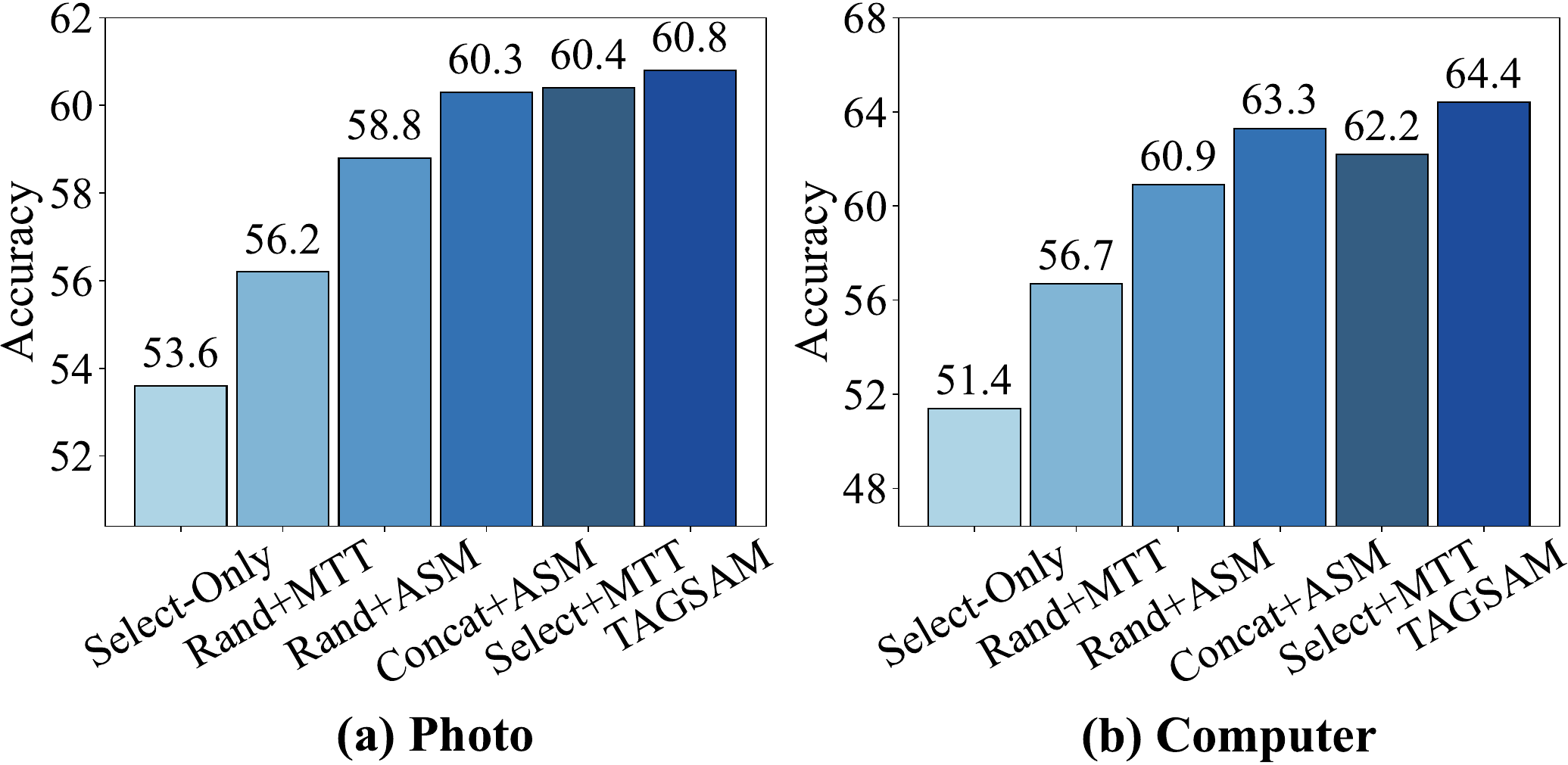}
        \caption{The ablation study for different variants. The $y$-axis is anchored at the accuracy of the baseline \textit{Random}.
        }
        \label{fig:ablation}
        \vspace{-3mm}
    \end{figure}

\subsection{Ablation Study}
    As reported in Figure~\ref{fig:ablation}, we experiment under the setting of synthetic dataset size \(M=200\), with 5 variants as follows:
    
    \squishlist
        \item \textit{Select-Only} conducts subgraph text selection to compress TAG.
        
        \item \textit{Rand+MTT} randomly selects texts and compresses graph via matching training trajectories. 
        
        \item \textit{Rand+ASM} randomly selects texts and compresses graph via our attribute similarity matching.
        
        \item \textit{Concat+ASM} samples subgraphs and directly concatenates all the texts within each subgraph to form a new text, then compresses graph via attribute similarity matching. 
        
        \item \textit{Select+MTT} compresses the text and graph by our subgraph text selection and matching training trajectories respectively.
    \squishend
    
    \stitle{The efficacy of text selection.} 
        Comparing ``Select-Only'' to the ``Random'' baseline confirms that text selection enriches the text modality for improved training. The comparison between ``Rand+ASM'' and ``Concat+ASM'' also reveals that text concatenation can enhance performance. However, this naive strategy is outperformed by ``TAGSAM'' as it introduces significant redundancy and noise, limiting its performance. In contrast, our method's precise selection mechanism yields superior performance, improving the ASM framework by 2.8\%. This improvement is more pronounced for the MTT strategy, reaching 4.9\%. These results demonstrate the significance and robustness of our subgraph text selection approach.
    
    \stitle{Attribute vs. Trajectory Matching.} 
        As illustrated in Figure~\ref{fig:ablation}, a clear performance advantage of our proposed Attribute Similarity Matching (ASM) is evident. A direct comparison between ``Rand+MTT'' and ``Rand+ASM'', as well as between ``Select+MTT'' and our ``TAGSAM'' model, consistently demonstrates that ASM significantly and reliably outperforms the Matching Training Trajectories (MTT) baseline across different text compression.

\section{Related Work}
\stitle{Selection-based condensation.}
Coreset selection is a critical technique that compresses large datasets into smaller subsets by selecting important samples. There are various pioneer literature reduces the computational burden of training on large datasets while maintaining model performance~\cite{welling2009herding, farahani2009facility, braverman2022power}. Main approaches are random sampling, Herding, which aligns subset and dataset means~\cite{welling2009herding}, and K-center, which minimizes the worst-case point-to-center distance~\cite{farahani2009facility}. All simply select existing data rather than synthesizing new samples; consequently, their performance is often sub-optimal.

\stitle{Graph dataset condensation.} Graph condensation aims at efficient GNN training \cite{acceleration, towards, edgae} by condensing a large graph into a small synthetic one without significant performance loss. Recently, numerous dataset condensation methods are developed for compressing graph datasets~\cite{gc1, gc2, exgc, gc3, gc4, self-supervised}. GCond~\cite{jin2021graph} matches training gradients between the original and condensed graphs. GCDM~\cite{liu2022graph} forces the spatial distribution of receptive fields in the condensed graph to mirrors that of the original dataset, i.e., distribution matching. However, these existing methods cannot handle the multi-modal setting of TAG condensation and rely on class-wise matching, making them heavily label-dependent.

\stitle{Multi-modal dataset condensation.} Unlike uni-modal methods, multi-modal apporaches compress multiple modalities simultaneously. MTT-VL~\cite{wu2023multimodal} first condense large-scale image-text datasets by matching training trajectories (MTT) into smaller dataset while maintaining effective training. LoRS~\cite{xu2024low} introduces a plug-and-play method that compresses image-text datasets by providing ground truth similarity between synthetic samples and reducing complexity with low-rank approximations. AVDD~\cite{kushwahaaudio} condenses audio-visual data by aligning modalities through implicit cross-matching. However, these methods cannot compress discrete text modalities. and such MTT-based methods often prove unstable in contrastive learning settings, resulting in sub-optimal performance.

\section{Conclusion}
In this work, we introduce TAGSAM, a novel dataset condensation method for Text-Attributed Graphs (TAGs) that addresses the issue that existing methods cannot effectively compress TAGs. Our method utilizes subgraph text selection to compress multiple texts into a representative and discrete one and attribute similarity matching to assign robust synthetic attributes. Experiments show that models trained on graphs condensed by TAGSAM to just 1\% of their original size still maintain competitive performance, thus making efficient pre-training on large-scale TAGs feasible.

\begin{acks}
This work was supported by the National Natural Science Foundation of China (No. 62472327, 62441229), the Fundamental Research Funds for the Central Universities (No. 2042025kf0040), and the Key R\&D Program of Hubei Province (No. 2023BAB077).
\end{acks}

\bibliographystyle{ACM-Reference-Format}
\balance
\bibliography{reference}

\appendix
\section{Proof of Theorem 1}
\label{apx:proof1}
\begin{proof}
Firstly, consider the impact of including \(c_i\) in \(\hat{T_o}\), where \( \hat{T_o} \subseteq D_o \) and \( c_i \in D_o \setminus \hat{T_o} \). The contribution of $c_i$ is calculated as: 
\[\sum_{c_j \in D_o} {\tt CMI}(c_i, c_j) - \sum_{c_j \in \hat{T_o}} {\tt CMI}(c_i, c_j)=\sum_{c_j \in D_o \setminus \hat{T_o}} {\tt CMI}(c_i, c_j).\]
As the estimated CMI values are non-negative, this term is clearly positive. This demonstrates that \( {\tt RS} \) is monotonic, implying that 
\[{\tt RS}(\hat{T_o} \cup \{c_i\}) \geq {\tt RS}(\hat{T_o}).\]
 
Secondly, let \( A \subseteq B \subseteq D_o \) and \( c_i \in D_o \setminus (A \cup B) \). The improvement of adding \( c_i \) to both \( A \) and \( B \) is characterized by the same information coverage gain \(\sum_{c_j \in D_o} {\tt CMI}(c_i, c_j)\). However, the redundancy penalty is larger for \( B \) because the difference 
\[\sum_{c_j \in B \setminus A} {\tt CMI}(c_i, c_j) = \sum_{c_j \in B} {\tt CMI}(c_i, c_j) - \sum_{c_j \in A} {\tt CMI}(c_i, c_j)\] 
is positive. Therefore, submodularity is satisfied:
\[
{\tt RS}(A \cup \{c_i\}) - {\tt RS}(A) \geq {\tt RS}(B \cup \{c_i\}) - {\tt RS}(B).
\]
Finally, according to Nemhauser's theorem~\cite{nemhauserAnalysisApproximationsMaximizing1978}, as \( {\tt RS}(\hat{T_o}) \) is monotonic and submodular, applying a greedy algorithm to maximize \( {\tt RS}(\hat{T_o})  \) under a cardinality constraint \( k \) guarantees that
     \[
        {\tt RS}(\hat{T_o}) \geq \left(1 - \frac{1}{e}\right) {\tt RS}(T^*_o).
     \]

\end{proof}

\section{Proof of Theorem 2}
\label{apx:proof2}
\begin{proof}
Following Zhang et al.~\cite{zhang2023generalization}, we adopt the TAG contrastive loss for ease of analysis:
\begin{equation}
\label{eq:tag_contrastive}
\begin{split}
\mathcal{L}_{\mathrm{con}}
 = {} & -2 \,\mathbb{E}_{X_i,T_i}\!\bigl[\,f_G(X_i)^{\top}f_T(T_i)\bigr] \\
       & +  \mathbb{E}_{X_i,T_j}\!\bigl[\,(f_G(X_i)^{\top}f_T(T_j))^{2}\bigr],
\end{split}
\end{equation}
where \( (X_i,T_i) \) denotes a positive graph-text pair and \( (X_i,T_j) \) with \( j\neq i \) denotes an independently sampled negative pair.

Let \(P\in\mathbb{R}^{N\times N}\) be the joint-probability matrix whose \((i,j)\)-entry is:
\[
P_{ij}\;=\;p_{G,T}(i,j).
\]
Define the marginal probabilities:
\[
p_G(i)=\sum_{j}p_{G,T}(i,j),\qquad
p_T(j)=\sum_{i}p_{G,T}(i,j),
\]
and form the \emph{normalized} co-occurrence matrix:
\[
\tilde{P}_{ij}\;=\;\frac{p_{G,T}(i,j)}{\sqrt{p_G(i)\,p_T(j)}}.
\]

Consider the asymmetric low-rank factorization loss:
\begin{equation}
\label{eq:l_amf_new}
\mathcal{L}_{\mathrm{AMF}}(F_G,F_T)
\;=\;
\bigl\lVert\,\tilde{P}-F_G F_T^{\top}\bigr\rVert_F^{2},
\end{equation}
where \(F_G\in\mathbb{R}^{N\times k}\) and
\(F_T\in\mathbb{R}^{N\times k}\) with \(k\ll N\).

Let the \(i\)-th row of \(F_G\) and the \(j\)-th row of \(F_T\) represent the encoded features of samples \(X_i\) and \(T_j\), respectively, in the following form:
\begin{subequations}
\label{eq:rows_new}
\begin{align}
F_G(i) &= \sqrt{p_G(i)}\,f_G(X_i)^{\top},\\
F_T(j) &= \sqrt{p_T(j)}\,f_T(T_j)^{\top}.
\end{align}
\end{subequations}

Substituting Equation~\ref{eq:rows_new} into Equation~\ref{eq:l_amf_new} and expanding the Frobenius norm yields:
\[
\begin{aligned}
\mathcal{L}_{\mathrm{AMF}}
&=\sum_{i,j}\!\Bigl(
\tilde{P}_{ij}-\sqrt{p_G(i)}\,f_G(X_i)^{\top}\sqrt{p_T(j)}\,f_T(T_j)
\Bigr)^{2} \\[2pt]
&=\sum_{i,j}\frac{p_{G,T}(i,j)^{2}}{p_G(i)p_T(j)} 
   -2\,\mathbb{E}_{X_i,T_j}\bigl[f_G(X_i)^{\top}f_T(T_j)\bigr] \\
   &+\mathbb{E}_{X_i,T_j}\bigl[\bigl(f_G(X_i)^{\top}f_T(T_j)\bigr)^{2}\bigr].
\end{aligned}
\]

The first summation depends only on the data distribution, so denote it by the constant \(C\).
Recognizing the remaining two terms as the contrastive loss \(\mathcal{L}_{\mathrm{con}}\) in
Eq.~\eqref{eq:tag_contrastive}, we obtain:
$$
\mathcal{L}_{\mathrm{AMF}}(F_G,F_T)
\,=\,
\mathcal{L}_{\mathrm{con}}(f_G,f_T)+C.
$$

Because the additive constant \(C\) is independent of the embeddings, minimizing the matrix factorization loss is equivalent to minimizing the contrastive loss.

By the classical Eckart--Young theorem, the best rank-$k$ approximation of the normalized co-occurrence matrix $\tilde P$ is its truncated SVD $U^{k}\Sigma^{k}V^{k\top}$.  Every pair of factor matrices that attains this optimum can therefore be written as:
$$
F_G^{\star}=U^{k}DR,\qquad  
F_T^{\star}=V^{k}\Sigma^{k}D^{-1}R,
$$
where $R\in\mathbb{R}^{k\times k}$ is an arbitrary unitary matrix and $D$ is an arbitrary invertible diagonal matrix.  Mapping the rows of $F_G^{\star}$ and $F_T^{\star}$ back to the encoder outputs yields:
$$
f_G^{\star}(X_i)=\frac{1}{\sqrt{p_G(i)}}\bigl(U^{k}_{i:}DR\bigr)^{\!\top}, 
$$
$$
f_T^{\star}(T_j)=\frac{1}{\sqrt{p_T(j)}}\bigl(V^{k}_{j:}\Sigma^{k}D^{-1}R\bigr)^{\!\top}.
$$

Hence the optimal graph and text embeddings are uniquely determined up to a rotation $R$ and scaling $D$, as claimed in Theorem 2.

Finally, taking the inner product of the optimal embeddings and applying the logarithm gives:
$$
\begin{aligned}
\log s^{*}_{ij} 
      &= \log\!\bigl(f^{*}_G(X_i)^{\top}f^{*}_T(T_j)\bigr) \\[2pt]
      &= \log\!\frac{p_{G,T}(i,j)}{p_G(i)\,p_T(j)} = \log\!\frac{P(X_i,T_j)}{P(X_i)P(T_j)}\\
      &= \operatorname{PMI}(X_i,T_j).
\end{aligned}
$$

Thus the optimal similarity is exactly the point-wise mutual information between $X_i$ and $T_j$, which completes the proof of Theorem 2.

\end{proof}

\section{Zero-shot Classification}
\label{apx:zero-shot}
Following the G2P2~\cite{wen2023augmenting, generative}, in this paper, zero-shot classification aligns graph and text representations within a unified semantic space using contrastive learning. The architecture employs a graph encoder to generate an embedding for each node and a text encoder to transform text prompts describing potential classes—including those unseen during training—into corresponding class embeddings. For inference, the model classifies a node by calculating the cosine similarity between its embedding and all candidate class embeddings. The node is then assigned to the class with the highest similarity score, enabling effective classification into novel categories without any specific training examples for them.

\section{Dataset Details}
\label{apx:dataset detail}

\begin{table}[t]
    \centering
    \caption{Statistics of Text-Attributed Graph datasets.}
    \begin{tabular}{l|rrrr}
    \toprule 
    \textbf{Dataset} & \# \textbf{Node} & \# \textbf{Edge} & \# \textbf{Avg.\ Deg} & \# \textbf{Class} \\ \midrule
    Cora             & 2,708            & 5,429           & 4.01                  & 7  \\
    Photo            & 48,362           & 500,928         & 20.72                 & 12 \\
    Computers        & 87,229           & 721,081         & 16.53                 & 10 \\
    ArXiv            & 169,343          & 1,166,243       & 13.77                 & 40 \\
    Products         & 2,449,029        & 61,859,140      & 50.52                 & 47 \\ \bottomrule
    \end{tabular}
    \label{tab:dataset statistics}
\end{table}

Table~\ref{tab:dataset statistics} brings together five widely adopted TAG benchmarks that span orders of magnitude in scale and density while covering both academic and e-commerce domains, ensuring that our evaluation stresses every component of the condensation pipeline.

Cora~\cite{sen2008collective}. This classic citation network contains 2,708 papers linked by 5,429 citation edges. Each node is accompanied by the raw abstract—typically 60 – 140 words—that spells out the problem statement, core methodology, and main results. 

Amazon Photo and Amazon Computers~\cite{shchur2018pitfalls}. Both graphs are extracted from the Amazon co-purchase network but focus on different product verticals. Amazon Photo (48 k nodes, 500 k edges) centers on cameras, lenses, lighting and accessories, whereas Amazon Computers (87 k nodes, 721 k edges) covers laptops, components and peripherals. For every product node we concatenate the seller’s catalogue title, the marketing blurb, and the first three customer-review snippets, yielding on average 35 – 50 tokens of plain-language description that mention brand names, use-cases, and distinguishing features. 

ArXiv (OGB-arXiv)~\cite{wang2020microsoft}. Drawn from the Open Graph Benchmark, this citation graph contains 169,343 preprints connected by 1.17 M directed edges. Every node stores the full paper title plus abstract (roughly 140 – 300 tokens), resulting in the longest and most technical text fields in our study. 

Amazon Products~\cite{hu2020open}. At 2.5 M nodes and 61 M edges, Amazon Products is the largest graph in our suite. Each node aggregates the product title, description, curated keyword tags, and representative review quotes. After basic cleaning this yields a median of 28 tokens and a long-tail distribution up to several hundred. 

\section{Additional Settings}
\label{apx:parameter}
For completeness, we report the exact hyper-parameter values used in all experiments.

\stitle{Subgraph text selection.}
    The subgraph size \(g = 4\), compression ratio \(k /\left| D \right| =0.6\). And GPT-2~\cite{radford2019language} is used as the language model \(\phi\).

\stitle{Attribute similarity matching.}
    The learnable step-size parameter \(\alpha\) is initialized to $5e-3$. Optimization is carried out for 5,000 epochs with a mini-batch size of \(m = 2000\) using stochastic gradient descent (SGD) with momentum 0.9 \cite{ruder2016overview}. The learning rates for \(\alpha\) and for the synthetic node attributes are $2e-6$ and 100, respectively. To facilitate training, we utilize a pre-trained BERT model and keep it frozen and train the teacher model for 15 epochs. 

\stitle{Experiment environment.}
    All experiments were conducted on a workstation running Ubuntu 22.04.3 LTS. The system was equipped with an Intel Core i9-13900K CPU, 64 GB of DDR5 RAM, and two NVIDIA GeForce RTX 4090 GPUs, each with 24 GB of VRAM. Our models were implemented in Python 3.11 using the PyTorch 2.5 framework, with CUDA 12.1 utilized for GPU acceleration.

\section{Limitation}
While subgraph text selection is effective for compression, the resulting merged texts may sometimes lack human-interpretable semantic coherence. Additionally, the process of aligning attribute similarity matrices, which involves complex structural relationships, highlights the need for further research into more computationally efficient matching loss functions.

\end{document}